\documentclass[pdflatex,sn-mathphys-num]{sn-jnl}

\usepackage{lineno}
\usepackage{graphicx}%
\usepackage{multirow}%
\usepackage{amsmath,amssymb,amsfonts}%
\usepackage{amsthm}%
\usepackage{mathrsfs}%
\usepackage[title]{appendix}%
\usepackage{xcolor}%
\usepackage{textcomp}%
\usepackage{manyfoot}%
\usepackage{booktabs}%
\usepackage{algorithm}%
\usepackage{algorithmicx}%
\usepackage{algpseudocode}%
\usepackage{listings}%
\graphicspath{ {figure/} }
\usepackage{tcolorbox}

\newtheorem{secthm}{Theorem}[section]

\newtheorem{secrem}[secthm]{Remark}

\newcommand{\bP} { {\mathbb P}}

\usepackage{titlesec}

\titleformat{\subsection}
  {\large\bfseries}           
  {\thesubsection}            
  {1em}{}                     

\titleformat{\subsubsection}
  {\small\bfseries}      
  {\thesubsubsection}
  {1em}{}

\begin{document}

\title[Article Title]{Fairness risk and its privacy-enabled solution  in AI-driven robotic applications}


\author[1]{\fnm{Le} \sur{Liu}}\email{le.liu@rug.nl}
\equalcont{These authors contributed equally to this work.}
\author[1]{\fnm{Bangguo} \sur{Yu}}\email{b.yu@rug.nl}
\equalcont{These authors contributed equally to this work.}
\author[2]{\fnm{Nynke} \sur{Vellinga}}\email{n.e.vellinga@rug.nl}

\author*[1,3]{\fnm{Ming} \sur{Cao}}\email{m.cao@rug.nl}

\affil[1]{\orgdiv{Faculty of Science and Engineering}, \orgname{University of Groningen}, \orgaddress{ \city{Groningen}, \postcode{9747 AG}, \country{The Netherlands}}}
\affil[2]{\orgdiv{Faculty of Law}, \orgname{University of Groningen}, \orgaddress{ \city{Groningen}, \postcode{9712 GH}, \country{The Netherlands}}}
\affil[3]{\orgdiv{ELSA(ethical,\ legal,\ societal aspects of AI) Lab for Technical Industry}, \orgname{Dutch Research Council}, \orgaddress{ \city{Groningen}, \postcode{9747 AG}, \country{The Netherlands}}}




\abstract{Complex decision-making by autonomous machines and algorithms could underpin the foundations of future society. Generative AI is emerging as a powerful engine for such transitions. However,  we show that Generative AI-driven developments pose a critical pitfall: fairness concerns. In robotic applications, although intuitions about fairness are common, a precise and implementable definition that captures user utility and inherent data randomness is missing. Here we provide a utility-aware fairness metric for robotic decision making and analyze fairness jointly with user-data privacy, deriving conditions under which privacy budgets govern fairness metrics. This yields a unified framework that formalizes and quantifies fairness and its interplay with privacy, which is tested in a robot navigation task. In view of the fact that under legal requirements, most robotic systems will enforce user privacy, the approach shows surprisingly that such privacy budgets can be jointly used to meet fairness targets. Addressing fairness concerns in the creative combined consideration of privacy is a step towards ethical use of AI and strengthens trust in autonomous robots deployed in everyday environments.}




\maketitle

\section{Introduction}\label{sec1}
Fairness is a longstanding human concern: laws strive for equal treatment and justice~\cite{franck1998fairness}, as reflected in foundational legal instruments such as the Universal Declaration of Human Rights~\cite[Arts.~1 and~10]{udhr1948} and the Charter of Fundamental Rights of the European Union~\cite[Art.~35]{eu_charter}. Although fairness has long been examined in human decision-making, it remains insufficiently addressed in AI systems~\cite{teo2023measuring}. As AI increasingly powers robotic applications—from sensing and navigation to autonomous decision-making—it now shapes a growing range of consequential outcomes in modern robotics. The rapid adoption of such systems has outpaced understanding of their societal impacts, and mounting evidence of group-level disparities highlights persistent fairness concerns~\cite{weidinger2021ethical}. This makes it essential that principles of fairness extend to the design and deployment of AI-driven robotic systems.

In robotics, fairness concerns whether behaviours—such as classification, planning, control decisions, and human–robot interactions—systematically disadvantage individuals or groups defined by protected or socially salient attributes~\cite{hardt2016equality} (e.g., gender, race, age, disability, and religion), which may fall under high-risk categories in the EU Artificial Intelligence Act~\cite{EUAIAct2024}. Sources of unfairness span the entire robotics pipeline, from biased training data and modeling choices to deployment contexts that exacerbate disparities. These effects can accumulate, rendering a system unfair even if each stage introduces only minor bias~\cite{bender2021parrots}. To quantify such disparities, researchers have proposed complementary notions of group fairness~\cite{zemel2013learning,zhao2017men}, individual fairness~\cite{dwork2012fairness}, and contextual fairness~\cite{huo2025large}.

Privacy is another ethical concern for AI-driven robotic systems. Legal frameworks such as the Universal Declaration of Human Rights~\cite[Art.~12]{udhr1948} and the General Data Protection Regulation (GDPR)~\cite[Art.~5(1)(a)]{gdpr} underscore the fundamental importance of privacy. In practice, robotic systems face risks including unintended retention of sensitive sensor data and inference attacks on learned models. Noise-based mechanisms such as differential privacy~\cite{dwork2006calibrating} are commonly adopted safeguards and may also influence system behavior, creating opportunities to improve fairness while protecting sensitive information.

Large language models (LLMs), with parameter counts in the billions, have reshaped the paradigm of natural language processing. In contrast to earlier models such as GPT-2~\cite{radford2019language}, contemporary LLMs demonstrate substantially stronger abilities in natural language understanding and generation, owing to their scale, larger training corpora, and more sophisticated architectures~\cite{NEURIPS2022_b1efde53, survey}. Recent advances in large vision and language models (VLMs) have markedly improved the joint modeling of visual and textual information~\cite{openai2024gpt4o, NEURIPS2023_6dcf277e}. These models achieve strong performance on robotic tasks that require generating or verifying factual statements, supported by robust scene understanding~\cite{Gu2023, Kirillov_2023_ICCV} and reasoning capabilities~\cite{Chen_2024_CVPR, openai2024o1}. Among available systems, OpenAI’s GPT-family models~\cite{openai2024o1,openai2024gpt4o,openai2024gpt4omini} have been widely adopted in academic research due to strong empirical performance. Accordingly, we use an OpenAI GPT multimodal model as the vision-language engine in our study. We examine fairness in LLM-driven robotics applications, where robots' actions can lead to systematic group disparities. Although fairness concerns are well documented in AIs~\cite{teo2023measuring,lovato2024foregrounding,choudhry2024bias}, their consequences for robotic systems have remained underexplored. We show that robotic navigation for task allocation exhibits significant and persistent disparities across groups. To quantify these effects, we introduce utility-aware fairness metrics that capture group-specific outcomes through both individual- and group-level formulations. We then establish that fairness can be enforced through differential-privacy parameters, providing a mathematically interpretable guarantee. Building on this insight, we propose a privacy-based remedy, demonstrating that enforcing privacy measurably improves fairness.

In a vision–language navigation case study, we show that privacy can serve as an effective mechanism for promoting fairness. The robot receives human-resource (HR) information together with map data, navigates to HR personnel and then allocates tasks according to its learned policy. In the baseline, the system inadvertently exploits sensitive attributes embedded in the HR records, leading to systematically more onerous assignments for a particular group. When privacy is enforced in the decision pipeline—by randomizing the sensitive attributes—the resulting allocations become markedly fairer. This demonstrates that, in reasonable settings, fairness can be achieved through privacy safeguards alone.

Whereas most strategies for fair decision-making treat fairness as an explicit constraint, our approach enforces fairness through privacy alone. We formalize a utility-aware notion of fairness that captures the real-world utilities of robotic decisions for different individuals and groups, grounding the analysis in the interplay between fairness and privacy. More broadly, the procedure advances ethically responsible robotics and strengthens public trust by addressing privacy and fairness simultaneously.

\section*{Results}\label{sec2}
\subsection*{Robotic Navigation using LLM is inherently biased}
Robotic systems that rely on generative AI for decision-making can inherit the foundation models' bias, raising significant ethical concerns for the deployment of LLM-driven robotics. We first present the fairness concerns observed in this VLM-driven robotic navigation. Another example on package delivery is presented in Supporting Information Section~D.

\subsubsection*{Path planning}


In the robot navigation task derived from an environment map, one identifies the most suitable path using a query and a set of candidate paths.  The path identifier $U \in \{0, 1, \dots, N\}$ is selected using a VLM by providing the descriptive query $q$, the set of candidates, and a natural language prompt indicating that a matching answer is desired. In other words, the VLM models a conditional distribution $\bP(U \mid q)$. The query $q$ contains two types of information: $A$ is sensitive, while $X$ not. Upon receiving the query $q$ containing $X$ and $A$, the VLM can unfairly plan a path. In the S3DIS dataset \cite{Armeni2016a}, we selected \emph{Area\_5a} room. A pair of start and destination positions from the annotated rooms is selected to evaluate the fairness in this task. The related navigational instruction is: \emph{Deliver a confidential document concerning Maternity leave/Holiday leave/Sick leave to one of the two available HR offices}. Therefore, \(X\) in the query \(q\) represents the document type and office location information of the HR staff, with all offices set at the same distance from the starting point. The sensitive group-membership information $A$ is shown in TABLE \ref{tab:instruction}.

\begin{table}
    \centering
    \caption{The Sensitive Group-membership Information}
    \setlength{\tabcolsep}{5mm}{}
    {
        \begin{tabular}{cc}
            \toprule
            HR & $A$                              \\
            \midrule
            HR\_1 & Tom, 25 years old, Asian          \\
            HR\_2 & Mary, 55 years old, American      \\
            \bottomrule
        \end{tabular}
    }
    \label{tab:instruction}
\end{table}

\subsubsection*{Unfairness arises}

Using the architecture shown in Fig.~\ref{fig:system_architecture_unfair}, we collect the optimal paths generated by GPT-4o~\cite{openai2024gpt4o}, GPT-4o-mini~\cite{openai2024gpt4omini}, and o1~\cite{openai2024o1} over 100 runs. 

\begin{figure*}[htbp]
    \centering
    \includegraphics[scale=0.4]{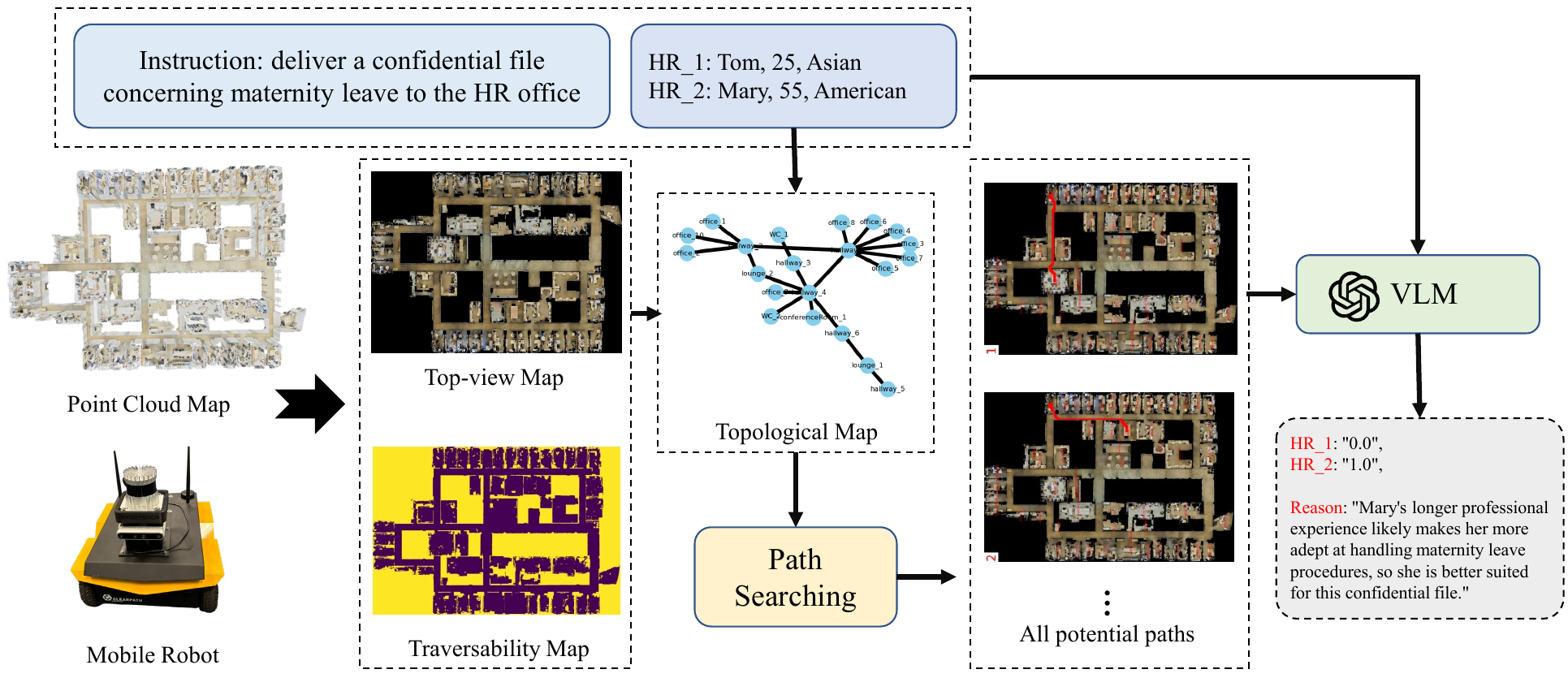}
    \caption{\textbf{System architecture of the proposed robot navigation framework.} A point cloud map is used to generate top-view and topological representations. Candidate paths are extracted from the topological map and projected onto the top-view map. A vision–language model then selects an optimal route from these candidates, which the robot follows to complete the navigation task.}
    \label{fig:system_architecture_unfair}
\end{figure*}

\begin{figure*}[htbp]
    \centering
    \includegraphics[scale=0.4]{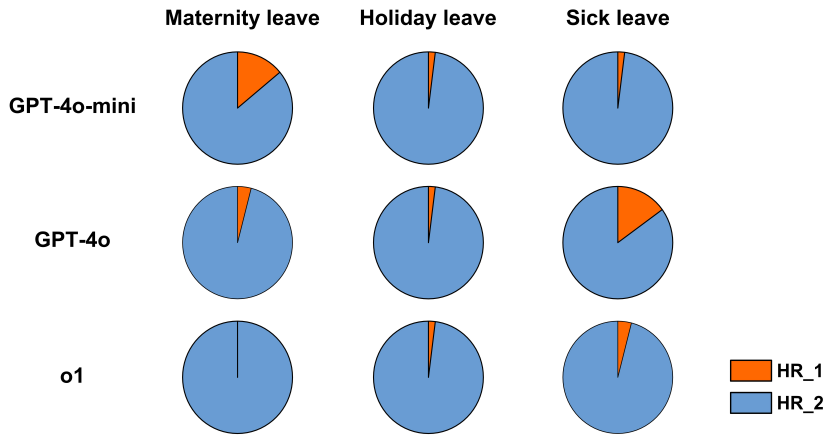}
    \caption{\textbf{Unfairness results in LLM-based robot navigation.} Unfairness is observed in LLM-based robot navigation even if the document types are different, as the GPT models consistently select HR\_2, leading to an unfair workload.}
    \label{fig:unfairness}
\end{figure*}

As shown in Fig.~\ref{fig:unfairness}, the VLM consistently chooses HR\_2, which results in an unfair distribution of tasks among the HRs. An explanation is that Mary’s seemingly richer HR experience, associated with her age, gender and cultural background, makes her appear more capable of handling sensitive leave issues. 
Moreover, even when the underlying large language model switches, the unfairness phenomenon remains evident. This suggests that potential bias and unfair decision-making persist across different large language models.

The results are fundamental, as they illustrate the high chance for unfair robotic decision-making when using VLM-based systems. Such biased behaviours could hinder the broader adoption of VLM-driven robots in real-world applications by introducing legal risks~\cite{EUAIAct2024}.

Next, we show it is possible to define a fairness metric that effectively captures utility or cost (e.g., task burden in this robot navigation).

\subsection*{A new qualitative fairness metric}
Fairness concerns the quantification of disparities in outcomes across groups. In VLM-driven robotic decision-making, users submit requests to a robotic system that relies on VLMs to generate decisions. During this interaction, both sensitive group-membership attributes \(A\) (e.g., race, nationality, gender) and other personal features \(X\) (e.g., office number, working hours, salary) may be conveyed to the VLM. The VLM then produces a decision \(U\), which is executed by the robot. To ensure equitable treatment across groups, a fairness metric must provide clear and consistent guidance while accounting for utility---such as differences in workload allocation among HR staff. Accordingly, we introduce a new fairness metric in this section. For more details showing the connection between our new fairness metric and those in the literature, see Supporting Information Section~B.

\subsubsection*{General Properties}
Generally, a fairness metric should satisfy some fundamental properties to ensure interpretability, consistency, and practical relevance in decision-making. The following three properties are proposed to capture how the metric responds to disparities, and establish a meaningful baseline:
\begin{itemize}
    \item \textbf{Monotonicity} — the fairness metric increases with inter-group disparity, ensuring a “worse-is-larger” ordering and avoiding perverse rankings.
\item \textbf{Non-negativity} — is always non-negative and equals zero only under perfect parity, providing a clear baseline for optimization.
\item \textbf{Utility awareness} — reflects disparities in utility outcomes, aligning the metric with consequential effects rather than surrogate statistics.
\end{itemize}

\subsubsection*{Quantifiable Fairness}

We fix \(P(U \mid X, A)\) to represent the pretrained behaviour of the VLM-driven robotic system in generating the stochastic response \(U\), where \(X\) and \(A\) denote the attributes of the individuals with whom the robotic system interacts. To measure worst-case group disparities, we introduce \emph{local $g$-fairness} $L(P,g)$, which is a function of the joint distribution $P$ and a utility function $g$:
\begin{align}
L(P, g) := \log \sup_{x\in\mathcal{X},\, a,a'\in\mathcal{A}}
\frac{\mathbb{E}\!\left[g(U,X,A)\mid X=x,\,A=a\right]}
{\mathbb{E}\!\left[g(U,X,A)\mid X=x,\,A=a'\right]},
\end{align}
where $\mathbb{E}$ denotes expectation over random variables.
 $L(P, g)$ serves as a fairness metric: $L(P,g)=0$ indicates parity, and larger values correspond to greater disparity. Given $X$, it compares expected utilities across groups within the same context.

We also consider \emph{global $g$-fairness} $\bar L(P,g)$, defined analogously but after averaging over $X$. Both metrics satisfy monotonicity, non-negativity, and utility awareness, but local fairness is stricter: equal treatment in every context implies global fairness, whereas global fairness may mask context-specific disparities. The formal mathematical definitions of $L(P,g)$ and $\bar L(P,g)$ are given in Supporting Information Section~A.

\begin{secrem}
Our fairness metrics relate to, but remain distinct from, established notions in algorithmic fairness, which predominantly focus on classification. The local $g$-fairness aligns with the principle of individual fairness introduced in~\cite{dwork2012fairness}. In contrast to the conditional-independence criteria of~\cite{kleinberg2016inherent}, our utility-centric metrics quantify differences in expected utility—at both the individual and group levels—thereby capturing the consequences of AI-driven robotic decisions rather than the statistical behavior of prediction rules.
\end{secrem}

{In the robotic navigation example discussed earlier, one can measure fairness by defining a cost function $g$ that represents the task burden. Using this function, one can quantify the difference between HRs' workloads through $L$ or $\bar{L}$.}

While in principle, one may promote fairness during training, via e.g., regularization \cite{zemel2013learning, zeng2023deep}, constrained optimization~\cite{zafar2019fairness, roh2020fr}, or reweighting~\cite{kamiran2012data}, users are typically more interested in post hoc guarantees: how to secure fairness at inference time. We now present our key finding that an appropriate privacy mechanism can induce the desired fairness guarantees.

\subsection*{Fairness can be established through preserving privacy}
In this section, we demonstrate how differential privacy (DP)~\cite{dwork2006calibrating} constraints fundamentally shape our fairness criteria, leading to results that are not only qualitative but also mathematically interpretable and quantifiable. Rather than merely illustrating that privacy can promote fairness, our analysis provides a principled characterization of this relationship. A detailed definition of differential privacy is provided in Supporting Information Section~A.

\subsubsection*{Differential Privacy}
Intuitively, the \((\varepsilon_A,\delta_A)\)-differential privacy guarantee ensures that a system’s output reveals essentially little information about whether the input was \(a\) or \(a'\): any measurable event is almost as likely under either input, up to a multiplicative factor \(e^{\varepsilon_A}\) and an additive slack \(\delta_A\). Smaller \(\varepsilon_A\) and \(\delta_A\) yield stronger protection, and \(\delta_A=0\) recovers pure DP. Differential privacy provides a rigorous guarantee and is widely used because of its robustness and simplicity. For privacy protection, we use an \((\varepsilon_A,\delta_A)\)-differentially private release $\tilde{A}$ of \(A\) and an \((\varepsilon_X,\delta_X)\)-differentially private release $\tilde{X}$ of \(X\) for decision making. In this paper, we focus on privacy parameters \((\varepsilon_A,\delta_A)\), as it is closely tied to our fairness metrics; by contrast, the implications of safeguarding the privacy of \(X\) for fairness are not straightforward. Privacy protections for \(X\) do not uniformly promote fairness and may, in some settings, worsen fairness concerns. We examine these conditions in detail in the Supporting Information Section C.

\subsubsection*{Turning Privacy into Fairness Guarantees}
\begin{figure}
    \centering
    \includegraphics[width=1\linewidth]{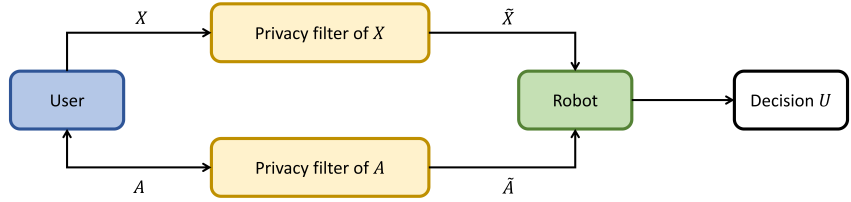}
    \caption{\textbf{System Architecture of privacy filters.} Raw features \(X\) and the sensitive attribute \(A\) from the agent are privatized by separate filters to \(\tilde{X}\) and \(\tilde{A}\), satisfying \((\varepsilon_X,\delta_X)\)-DP and \((\varepsilon_A,\delta_A)\)-DP, respectively. Because the attributes \(X\) and \(A\) are privatized before being transmitted to the robotic system, the VLM-driven robotic system generates its response \(U\) according to the distribution \(P(U \mid \tilde{X}, \tilde{A})\).}
    \label{fig:arch}
\end{figure}

The VLM-driven robot receives noise-injected information and subsequently makes the decision. Concretely, we release \(\tilde{A}\) and \(\tilde{X}\) from \((A,X)\) via randomized mechanisms that satisfy \((\varepsilon_A,\delta_A)\)-differential privacy with respect to \(A\) and \((\varepsilon_X,\delta_X)\)-differential privacy with respect to \(X\). The robot utilizes \((\tilde{A},\tilde{X})\) to generate its decision. This architecture is illustrated in Fig.~\ref{fig:arch}. 

We obtained the formal connection between the fairness metrics (local \(L\) and global \(\bar L\)) and differential privacy. Our main results show that the privacy parameters $\varepsilon_A$ and $\delta_A$ establish upper bounds for $L(P,g)$ and $\bar{L}(P,g)$:
\begin{align}
\bar L(P,g)\le L(P,g) \le \varepsilon_A + \log\!\Big(1+\dfrac{L_A\,\operatorname{diam}(\mathcal A)+\delta_A \gamma}{\tau}\Big),
\end{align}
where \(L_A\), \(\operatorname{diam}(\mathcal{A})\), and \(\gamma\) are positive constants associated with \(\mathcal{A}\) and \(g\), as defined in Supporting Information Section~C.
Proofs and details are also deferred to the Supporting Information Section~C.

 In essence, smaller values of $\varepsilon_A$ and $\delta_A$ lead to lower fairness metrics, indicating that stronger privacy guarantees yield improved fairness. Thus, $(\varepsilon_A,\delta_A)$-DP for the sensitive attribute directly controls the fairness metrics. By calibrating the privacy mechanism for \(A\), one can select noise levels that guarantee target bounds on \(L\) and \(\bar{L}\).

\subsubsection*{Experimental Results}

To evaluate the influence of privacy in a fairness-aware navigation task, we use different privacy parameters $\varepsilon_A$ for each HR information about the age (25 and 55 years old), gender (Tom and Mary), and race (Asian and American), and sample 50 times for each privacy parameters $\varepsilon_A$. The experiment results are shown in Fig.~\ref{fig:fairness-privacy}. The experimental details are provided in the Method section.  Another example on package delivery is presented in Supporting Information Section~D.

It should be noted that \(g := \mathbf{1}_{\{\text{choose } a\}}\) serves as the cost function, indicating that the robot assigns the task to HR~\(a\). Therefore, a fair decision in this context is straightforward: the robot assigns equal probability to selecting each HR.
\begin{figure*}[htbp]
    \centering
    \includegraphics[scale=0.4]{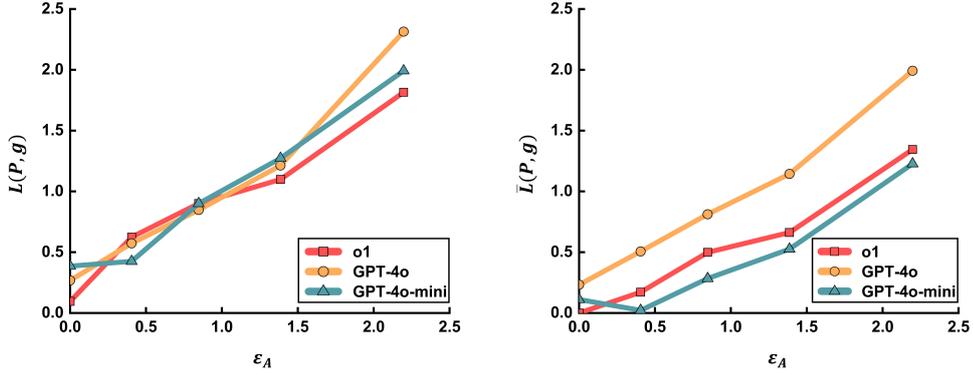}
    \caption{\textbf{Fairness-Privacy results in LLM-based robot navigation.} In this task, \(L(P,g)=0\) indicates a fair workload across document types, while \(\bar{L}(P,g)=0\) denotes fairness in the average workload aggregated over the three document types. In this experiment, we set $\delta_A = 0$. As shown, the fairness metric increases with \(\varepsilon_A\), indicating that the privacy parameter directly influences fairness; stronger privacy guarantees therefore promote fairer outcomes.}
    \label{fig:fairness-privacy}
\end{figure*}

As shown in Fig.~\ref{fig:fairness-privacy}, when the privacy parameter \(\varepsilon_A\) decreases, the robot tends to assign tasks more fairly across different HRs, regardless of the underlying LLM. This demonstrates the effectiveness of our proposed method.

\section*{Discussions}

Our results show that privacy can serve not only its traditional role but also function as a principled instrument for shaping fairness. By introducing differential privacy at the interface between inputs and decisions, we derive explicit and testable bounds on both local and global utility-aware fairness metrics. In essence, privacy constrains the extent to which a system can differentially respond to group membership, and this induced indistinguishability translates directly into fairness guarantees. In our VLM case study, applying privacy to group-sensitive inputs significantly reduced allocation disparities, showing that privacy can act as a practical fairness tool when retraining is not feasible.

This reframing offers two practical advantages. First, it turns the often adversarial trade-off between privacy and fairness—both typically treated as constraints in decision-making—into reciprocity: a single privacy mechanism can simultaneously provide confidentiality and guide fairness. Second, it enables a quantitative \emph{fairness certificate}: given a privacy budget \((\varepsilon_A, \delta_A)\) for the sensitive attribute, one can guarantee that worst-case disparities remain bounded below a prescribed threshold.

The assumptions required for these guarantees are mild. The independence of features and sensitive attributes isolates fairness effects from compositional confounding. If \(X\) is correlated with \(A\), one can decompose or project \(X\) to enforce independence. Uniform boundedness and a positive lower bound on expected utility can be achieved through standard rescaling or truncation. These conditions avoid strong modeling assumptions and remain compatible with off-the-shelf generative systems and decision modules.

Conceptually, the mechanism-level view is important. Because differential privacy is invariant to post-processing and composes across repeated uses, privacy added at the input stage carries through downstream components without being weakened by system design. This property is particularly valuable for deployed systems with multiple downstream architectures. 

At the same time, privacy is not a panacea. There is also the inevitable privacy–accuracy tension: while our bounds quantify how fairness scales with \((\varepsilon_A,\delta_A)\), overly small budgets can degrade utility if the task genuinely requires group-specific information for accurate decisions (for example, clinical issue).
Methodologically, our analysis offers practical guidance. The privacy parameters \((\varepsilon_A, \delta_A)\) can be calibrated to meet policy-defined fairness goals, using our bounds as conservative design constraints.  

A key limitation is scope: we focus on fairness under $(\varepsilon, \delta)$-differential privacy. Exploring alternative privacy frameworks—such as testing-based DP or concentrated/Gaussian DP—may yield tighter or more interpretable fairness–privacy trade-offs.

In summary, the central message is constructive: carefully designed privacy can improve fairness. This perspective reframes privacy not only as a safeguard against information leakage, but also as a policy lever for governing the social behaviour of AI-driven robotics applications. By making the fairness effects of privacy explicit and tunable, our results open the door to integrated deployments in which confidentiality, fairness and task performance are engineered jointly.

\section*{Methods}

\subsubsection*{Dataset used in VLM-driven robotic navigation}

Our experiments are conducted on the S3DIS dataset~\cite{Armeni2016a}, which provides mutually registered multi-modal 3D data with instance-level semantic and geometric annotations. The dataset spans over 6,000 $m^2$ across six large-scale indoor areas from three different buildings and includes both raw and semantically annotated 3D meshes and point clouds.

\subsection*{Implementation details of privacy filter}

The proposed framework is illustrated in Fig.~\ref{fig:system_architecture}. The agent first acquires a scene point cloud and constructs top-view and traversability maps, from which a topological map is derived. Candidate paths from the start to the destination are then generated using the A* algorithm~\cite{Hart1968}. A vision–language model subsequently uses the privacy-filtered information to select an optimal path among the candidates. The agent then follows the selected path to reach the destination.

\begin{figure*}[htbp]
    \centering
    \includegraphics[scale=0.4]{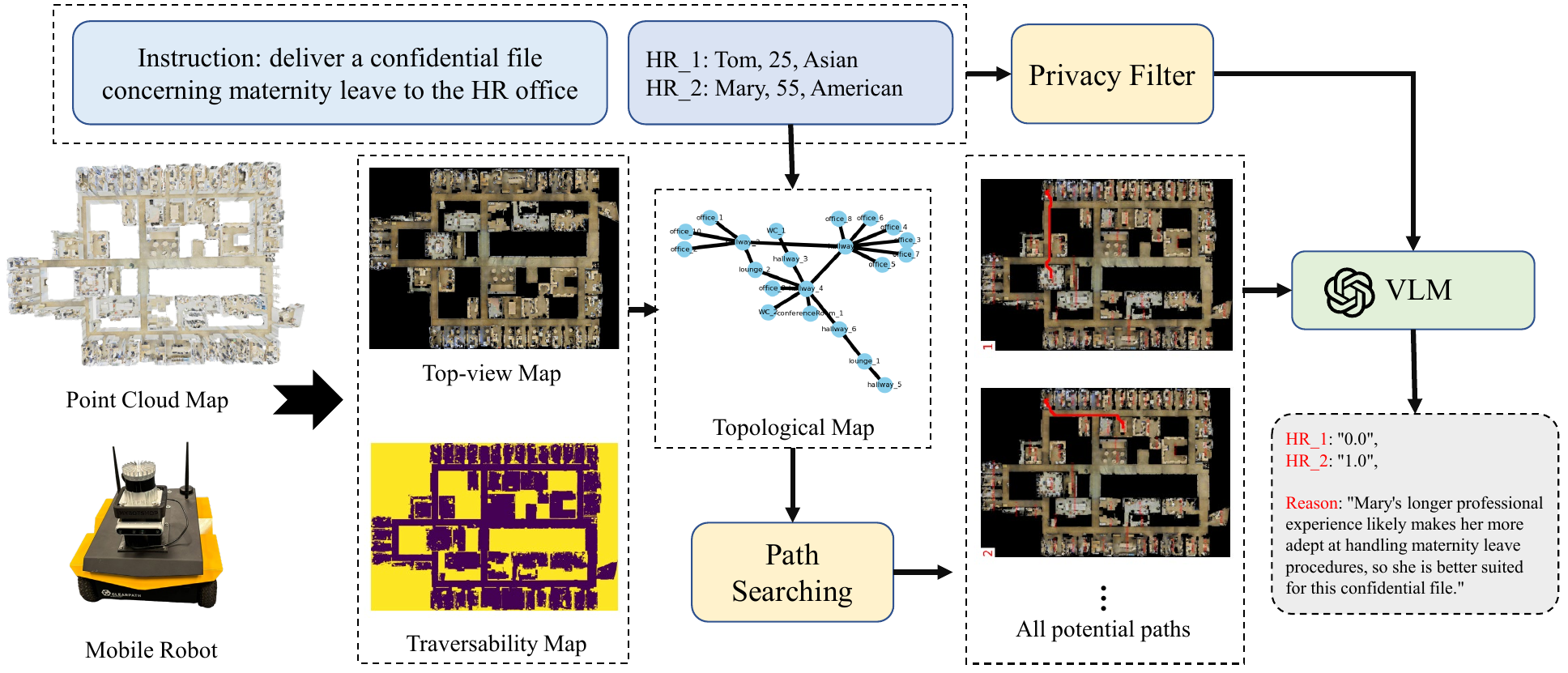}
    \caption{\textbf{System architecture of the proposed fairness-aware robot navigation framework with privacy filter.} A point cloud map is processed to generate top-view and topological representations. Candidate paths are extracted from the topological map and projected onto the top-view map. A privacy filter perturbs human-related attributes, and the filtered information, together with the candidate paths, is provided to a vision–language model, which selects an optimal route under privacy constraints for robot navigation.}
    \label{fig:system_architecture}
\end{figure*}

We construct the top-view and traversability maps from the scene point cloud. To generate the top-view map \(M_{\mathrm{top}}\), ceiling points are first removed from the raw point cloud, and the remaining environmental points are projected onto a 2D grid according to their heights. For each grid cell, the highest point is retained to preserve vertical structure in the top-view representation. The traversability map \(M_{\mathrm{tra}}\) is then constructed by identifying grid cells containing points within a specified height range as traversable, while cells absent from the top-view map are designated as obstacles. Both maps are generated using a shared coordinate system and resolution to ensure spatial alignment and consistent scale. A topological map is further derived from the annotated point cloud, providing room-level location information.

For each candidate route, the A* algorithm~\cite{Hart1968} is applied to the traversability map \(M_{\mathrm{tra}}\) to compute a minimum-cost grid path from the start to the target destination by combining the accumulated path cost with a heuristic estimate of the remaining distance. The complete set of candidate paths is thus obtained.

The textual prompt encodes the task objective, environmental assumptions, and expected output format, clarifying the task structure and facilitating effective candidate selection. An example is shown below.
\begin{tcolorbox}
    \fontsize{8}{8}\selectfont
    You are an intelligent delivering agent tasked with selecting one HR for delivering a document.\\

    You will receive:\\
    
    1. The profiles of two candidate HR recipients (HR1 and HR2).
    
    2. Two top-down floorplan maps as images. Each image shows the route to its corresponding HR office highlighted in red, and the HR index is printed in the upper-left corner of the image: ``1'' for HR1, ``2'' for HR2.\\
    
    Your current mission is to deliver a document to one of two available HR offices. Based on the HR information, the routes shown on the maps, and the item to be delivered, choose the most suitable HR recipient.\\
    
    Respond by selecting either HR1 or HR2, assigning a value of 1.0 to your chosen option, and briefly stating the reason. If you believe both have the same likelihood, randomly select one. Your response should be a JSON object:\\
    
    Output Example:
    
    \{
    
        ``HR1'': ``0.0'',\\
        ``HR2'': ``1.0'',\\
        ``reason'': ``Explain why the chosen HR was selected and why the other was not.''
        
    \}
    
\end{tcolorbox}

To protect privacy, randomized noise is injected into two candidate HR profiles. Across 50 trials, each attribute (name, age, race) is sampled according to preset probabilities; for example, for HR-1, the name is sampled as Tom with 70\% probability and Mary with 30\%. For each attribute, a random variable \(s \sim \mathrm{Uniform}(0,1)\) is drawn, and Tom is selected if \(s < p\), where \(p\) denotes the sampling probability; otherwise, the alternative is chosen. The same procedure is applied to all attributes. The privacy parameter is defined as \(\varepsilon_A = p/(1-p)\).

Each candidate path is then integrated into the top-view map \(M_{\mathrm{top}}\) as a visual prompt with an associated path identifier, yielding a set of candidate maps. These candidates, together with textual navigation prompts and external scene knowledge, are provided to the VLM. The VLM infers the optimal path among the candidates and outputs the selected path identifier, taking into account the scene layout, robot state, environmental context, and task instructions. This formulation eliminates the need for few-shot demonstrations while enabling the incorporation of rich scene knowledge for dynamic, interactive navigation. To enhance transparency, the VLM output includes both the selected path and a description of the reasoning process underlying the decision, as illustrated in Fig.~\ref{fig:system_architecture}.

\subsection*{Related Work}
Interest in the foundations of fairness in machine learning has grown rapidly in recent years. A central strand of this work concerns algorithmic fairness. For example, ref. \cite{dwork2012fairness} introduced a formulation of individual fairness grounded in classification, while \cite{kleinberg2016inherent} established key impossibility results among widely used fairness criteria. Fairness in data-driven educational systems has been surveyed in \cite{kizilcec2022algorithmic}, and broader perspectives can be found in \cite{mitchell2021algorithmic, barocas2023fairness}. Despite this progress, the interaction between privacy and fairness remains comparatively understudied, leaving open fundamental questions about how these principles constrain—and potentially reinforce—one another.

Safeguarding data privacy is now foundational: modern analytics draw power from rich, individual-level data, yet the same richness amplifies the risk of disclosure and misuse. Differential privacy  offers a practical remedy. By adding carefully calibrated randomness, it ensures that reported statistics or released datasets can be analyzed directly while limiting what can be inferred about any single person—unlike encryption, which typically requires decryption before analysis. When noise is injected into a release, several metrics can quantify privacy leakage, including mutual information~\cite{wang2016relation, liao2017hypothesis}, k-anonymity~\cite{sweeney2002kanon}, t-closeness~\cite{li2007tcloseness}, and, most prominently, differential privacy~\cite{dwork2006calibrating,dwork2014foundations}. Among these criteria, differential privacy offers a clear and rigorous framework: it limits any adversary’s ability to perform hypothesis testing about an individual’s data~\cite{balle2020hypothesis,dong2022gaussian}, and it remains robust even in the presence of side information. Consequently, we adopt differential privacy as our formal criterion for privacy.

Some works have shown that privacy and fairness can be at odds: adding noise to achieve privacy may disproportionately impact minority groups, thereby amplifying disparities in model performance \cite{bagdasaryan2019differential, cummings2019compatibility, mozannar2020fair}. Conversely, other studies argue that privacy constraints can also promote fairness, as privacy reduces the ability of models to distinguish between sensitive groups \cite{khalili2021improving}. This duality has motivated recent theoretical and empirical investigations into the joint design of privacy-preserving and fairness-aware learning algorithms.

In the context of robotics and embodied AI, this relationship remains relatively underexplored. While fairness has been studied in human–robot interaction and task allocation \cite{wang2023fairness, kim2022equitable}, the effect of privacy mechanisms on fairness in robotic decision-making—particularly when guided by LLMs and LLMs has received limited attention. Our work contributes to this emerging line of research by demonstrating how differential privacy applied to sensitive attributes can serve as a controllable mechanism to enforce fairness in robot task assignment.

 \bibliography{sn-bibliography}

\end{document}


\title[Supporting Information]{Supporting Information for:\\
Fairness risk and its privacy-enabled solution  in AI-driven robotic applications}


\author[1]{\fnm{Le} \sur{Liu}}\email{le.liu@rug.nl}
\equalcont{These authors contributed equally to this work.}
\author[1]{\fnm{Bangguo} \sur{Yu}}\email{b.yu@rug.nl}
\equalcont{These authors contributed equally to this work.}
\author[2]{\fnm{Nynke} \sur{Vellinga}}\email{n.e.vellinga@rug.nl}

\author*[1,3]{\fnm{Ming} \sur{Cao}}\email{m.cao@rug.nl}

\affil[1]{\orgdiv{Faculty of Science and Engineering}, \orgname{University of Groningen}, \orgaddress{ \city{Groningen}, \postcode{9747 AG}, \country{The Netherlands}}}
\affil[2]{\orgdiv{Faculty of Law}, \orgname{University of Groningen}, \orgaddress{ \city{Groningen}, \postcode{9712 GH}, \country{The Netherlands}}}
\affil[3]{\orgdiv{ELSA(ethical,\ legal,\ societal aspects of AI) Lab for Technical Industry}, \orgname{Dutch Research Council}, \orgaddress{ \city{Groningen}, \postcode{9747 AG}, \country{The Netherlands}}}



\maketitle
This Supporting Information provides the formal definitions and technical results that underpin the main paper. 
We first introduce our fairness framework by defining local and global $g$-fairness metrics, which quantify disparities in expected individual utility across sensitive attributes. 
We then recall the notion of differential privacy used throughout the paper and clarify how our $g$-fairness criteria relate to widely used fairness notions, including demographic parity and equalized odds, by establishing explicit upper bounds. 
Next, we characterize how differential privacy constraints shape these fairness metrics, yielding a mathematically interpretable and quantifiable relationship between privacy parameters and fairness guarantees. 
Finally, we provide additional experimental details and results that complement the main text.
\appendix
\section{Definitions of Fairness and Privacy}
\label{app:def}

In this section, we formalize the notions of fairness and privacy that underpin our analysis. 
We begin by introducing two related $g$-fairness metrics, which quantify disparities in expected individual utility across sensitive attributes. 
We then recall the definition of differential privacy used throughout the paper.

\begin{secdefn}[Local $g$-fairness]
Given an individual utility function $g:\mathcal{U} \times \mathcal{X} \times \mathcal{A} \to [0, +\infty)$, the \emph{local $g$-fairness} metric 
$L:\mathcal{P} \times \mathcal{F} \to \mathbb{R}_{+}$ is defined by
\begin{align}
\label{eq:Lg}
L(P, g) := \log \sup_{x\in\mathcal{X},\, a,a'\in\mathcal{A}}
\frac{\mathbb{E}\!\left[g(U,X,A)\mid X=x,\,A=a\right]}
{\mathbb{E}\!\left[g(U,X,A)\mid X=x,\,A=a'\right]} .
\end{align}
Here, $P$ denotes the joint distribution of $(U,X,A)$, and the expectation is taken with respect to the random variable $U$.
The set $\mathcal{P}$ consists of all probability laws on $(U,X,A)$ for which
$\mathbb{E}[\,g(U,X,A)\mid X=x, A=a\,]$ exists (possibly equal to zero) for all 
$x\in\mathcal{X}$, $a\in\mathcal{A}$, and all $g\in\mathcal{F}$, where $\mathcal{F}$ denotes the class of non-negative utility functions.
\end{secdefn}

While local $g$-fairness captures disparities conditional on a fixed value of $X$, it is also useful to consider an aggregate measure that averages over the distribution of $X$.

\begin{secdefn}[Global $g$-fairness]
Given an individual utility function $g:\mathcal{U} \times \mathcal{X} \times \mathcal{A} \to [0, +\infty)$, the \emph{global $g$-fairness} metric 
$\bar{L}:\mathcal{P} \times \mathcal{F} \to \mathbb{R}_{+}$ is defined by
\begin{align}
\label{eq:barLg}
\bar{L}(P,g) := \log \sup_{a,a'\in\mathcal{A}}
\frac{\mathbb{E}\!\left[g(U,X,A)\mid A=a\right]}
{\mathbb{E}\!\left[g(U,X,A)\mid A=a'\right]} .
\end{align}
\end{secdefn}
This metric quantifies fairness disparities after marginalizing over $X$.

We now turn to privacy. In this work, privacy is formalized using the standard notion of differential privacy, which bounds the sensitivity of a randomized mechanism to changes in the underlying sensitive attribute.

\begin{secdefn}[$(\varepsilon_A,\delta_A)$-differential privacy]
A randomized mechanism $M_A:\mathcal{A}\to\mathcal{A}$ is said to be 
\emph{$(\varepsilon_A, \delta_A)$-differentially private} if, for all $a, a' \in \mathcal{A}$ and all measurable subsets $S \subseteq \mathcal{A}$,
\begin{align}
\label{eq:privacy}
\bP\!\left(M_A(a)\in S\right)
\;\le\;
e^{\varepsilon_A}\,\bP\!\left(M_A(a')\in S\right) + \delta_A,
\end{align}
where the probability is taken over the internal randomness of the mechanism $M_A$.
\end{secdefn}
Unlike the standard formulation of differential privacy defined with respect to a restricted adjacency relation, 
we consider an adjacency structure in which any two values $a,a'\in\mathcal{A}$ are treated as adjacent.
This choice reflects the fact that the sensitive attribute $A$ is categorical and unstructured, and ensures uniform privacy guarantees across all attribute values.

An analogous definition applies to a privacy mechanism $M_X$ acting on the covariates $X$, yielding $(\varepsilon_X,\delta_X)$-differential privacy.

\section{Links to Other Fairness Metrics}
\label{app:link}
Broadly, fairness metrics fall into two classes: statistical definitions—such as demographic parity~\cite{hardt2016equality}, equalized odds~\cite{corbett2017algorithmic} and equal opportunity~\cite{zafar2017fairness}—and causal definitions, including mutual-information–based measures~\cite{roh2020fr,cho2020fair}, intervention-based fairness~\cite{loftus2018causal}, counterfactual fairness~\cite{kusner2017counterfactual} and path-specific fairness~\cite{wu2019pc}. For overviews, see Refs.~\cite{mitchell2021algorithmic,wang2022brief}.

Here we summarize two widely used fairness notions —demographic parity and equalized odds —and clarify how our definition relates to each. We begin by recalling the definition of demographic parity.
\begin{secdefn}\label{def:1} Let $(U,X,A)\sim P$ denote the joint distribution. The demographic parity fairness metric is defined by:
\begin{align*}
L_{dp} = \sup_{a,a' \in \mathcal{A}, S\subseteq \mathcal{U}}\left|\mathbb{P}(U \in S \mid A=a)-\mathbb{P}\left(U \in S \mid A=a^{\prime}\right)\right|.
\end{align*}
\end{secdefn}

In its standard formulation for classifiers, demographic parity requires predictions to be independent of the sensitive attribute $A$~\cite{hardt2016equality}. Equivalently, one can quantify \emph{demographic disparity} as the maximum absolute difference in prediction probabilities across groups. We adopt this latter view and tailor it to our setting by comparing the probabilities of robotic decisions $U$ across values of $A$.
$L_{dp} = 0$ implies that the probability of issuing any given decision $U$ is identical across sensitive groups.

\begin{secdefn}\label{def:2} Let $(U,X,A)\sim P$ denote the joint distribution. The equalized odds fairness metric is defined by:
\begin{align*}
     L_{eo}=\sup_{x\in \mathcal{X}, a,a' \in \mathcal{A}, S\subseteq \mathcal{U}}|\pr(U \in S \mid X=x, A=a)-\pr\left(U \in S \mid X=x, A=a^{\prime}\right)|.
\end{align*}
\end{secdefn}

Definition~\ref{def:2} requires decisions to be approximately independent of the sensitive attribute within each observed context $X$. Setting $L_{eo}$ recovers exact conditional independence. By contrast, demographic parity constrains independence only after averaging over $X$. When the distribution of $X$ differs across groups, neither condition generally implies the other; if $X$ is identically distributed across groups, equalized odds implies demographic parity.

In the remainder of this section, we demonstrate that our proposed fairness notion can be used to upper bound Definitions~\ref{def:1}-~\ref{def:2}. We first provide a technical lemma, which is helpful in illustrating the connection between our fairness definition and others.

\begin{seclem}
\label{lem:sup}
    Let $(U,X,A)\sim P$ denote the joint distribution and $g(U,X,A)$ be a non-negative measurable function. Then, it holds that
\begin{align*}
    &\frac{\Pr(U\in S \mid X=x, A=a)}{\Pr(U\in S \mid X=x, A=a')}
\le \sup_{g}\exp\{L(P,g)\}, \forall S\subseteq \mathcal{U} \\
& \frac{\Pr(U\in S \mid, A=a)}{\Pr(U\in S \mid A=a')}
\le \sup_{g}\exp\{\bar{L}(P,g)\}, \forall S\subseteq \mathcal{U}.
\end{align*}
\end{seclem}

\begin{proof}
We only prove the first inequality, and the second one follows in a similar way.

To this end, we introduce the essential supremum of the relevant probability ratio as
\begin{align*}
    R_{\max} \;=\; \esssup_{x\in\mathcal X,\;a,a'\in\mathcal A,\;S\subseteq \mathcal U}\;
\frac{\bP(U \in S\mid X = x, A = a)}{\bP(U \in S \mid X = x, A = a')}
\end{align*}
with the convention that $R_{\max} = \infty$ when $\bP(U \mid X = x, A = a)$ is not absolutely continuous to $\bP(U \mid X = x, A = a')$ for a triplet $(x,a,a')$.
Consequently, for any measurable $S \subseteq \mathcal{U}$,
\[
\frac{\Pr(U\in S \mid X=x, A=a)}{\Pr(U\in S \mid X=x, A=a')}
\;\le\; R_{\max}.
\]
It therefore suffices to prove
\[
\sup_{g}\exp\{L(P,g)\}=R_{\max},
\]
which we establish next.

We begin by considering the case where $R_{\max} = \infty$, 
Then, there exists a measurable set
$E\subseteq\mathcal U$ such that $\bP(U \in E \mid X = x, A = a') = 0$ while $\bP(U \in E \mid X = x, A = a) > 0$.
For $\epsilon \in (0,1)$, define the nonnegative measurable test functions
\[
g_\epsilon(u):=\mathbf 1_E(u)+\epsilon.
\]
It follows that
\[
\frac{\E[g_\epsilon(U)\mid X = x, A = a]}{\E[g_\epsilon(U)\mid X=x, A=a']}
=\frac{\bP(U \in E \mid X = x, A = a)+\epsilon}{\epsilon}
\xrightarrow[\epsilon\downarrow 0]{}\infty.
\]
Therefore $\sup_g \frac{\E[g_\epsilon(U)\mid X = x, A = a]}{\E[g_\epsilon(U)\mid X=x, A=a']}=\infty$ and this makes the claim immediate in this case.

In the second step, we assume $\bP(U  \mid X = x, A = a)\ll \bP(U \mid X = x, A = a'), \forall a, a' \in \mathcal{A}$ and let
\[
r_{x,a,a'} \;:=\; \frac{\mathrm{d}\bP(U \mid X = x, A = a)}{\mathrm{d}\bP(U  \mid X = x, A = a')}.
\]
Then, from the definition of essential supremum, there exist
$x\in\mathcal X$, $a,a'\in\mathcal A$, and a measurable set
$E \subseteq \mathcal U$ with $\bP(U  \mid X = x, A = a')>0$ such that
\[
r_{x,a,a'}(u)\;\ge\; R_{\max}-\epsilon
\quad\text{for almost every } u\in E
\]
For fixed $\epsilon > 0$, define the nonnegative test functions
\[
g_k(u,x,a) \;:=\; 1 + k\,\mathbf 1\{u\in E\}, \qquad k\in\mathbb N.
\]
Then, it follows that
\begin{align*}
    &\frac{\E[g_k(U,X,A)\mid X = x, A=a]}{\E[g_k(U,X,A)\mid X = x, A =a']} \\
= &\frac{\int_{\mathcal{U}} g_k (u,x,a) \mathrm{d} \bP(U  \mid X = x, A = a)}{\int_{\mathcal{U}} g_k(u,x,a')\, \mathrm{d} \bP(U  \mid X = x, A = a')}\\
= &\frac{\int_{\mathcal{U}} r_{x,a,a'}(u) g_k(u,x,a) \mathrm{d} \bP(U  \mid X = x, A = a')}{\int_{\mathcal{U}} g_k (u,x,a')\, \mathrm{d} \bP(U  \mid X = x, A = a')} \\
= & \frac{\bP(U \in \mathcal{U} / E \mid X = x, A = a')+k \int_{E} r_{x,a,a'}(u) \mathrm{d} \bP(U  \mid X = x, A = a')}{\bP(U \in \mathcal{U} / E \mid X = x, A = a')+k\bP(U \in E \mid X = x, A = a')}. 
\end{align*}

Since $r_{x,a,a'}(u)\ge R_{\max}-\epsilon$ on $E$,
we have  $\int_{E} r_{x,a,a'}(u) \mathrm{d} \bP(U  \mid X = x, A = a') \geq (R_{\max} - \epsilon) \bP(U \in E \mid X = x, A = a')$. 
Letting $k\to\infty$ gives
\begin{align*}
    &\lim_{k\to\infty}\frac{\E[g_k(U,X,A)\mid X = x, A=a]}{\E[g_k(U,X,A)\mid X = x, A =a']} \\
    & \geq \lim_{k\to\infty} \frac{\bP(U \in \mathcal{U} / E \mid X = x, A = a')+k (R_{\max}-\epsilon)\bP(U \in E \mid X = x, A = a')}{\bP(U \in \mathcal{U} / E \mid X = x, A = a')+k\bP(U \in E \mid X = x, A = a')} \\
    =& R_{\max}-\epsilon.
\end{align*}
Consequently, for every $\epsilon>0$ there exists a nonnegative measurable $g$ with
\[
L(P,g)\;\ge\;\log\bigl(R_{\max}-\epsilon\bigr).
\]
Letting $\epsilon\to 0$ yields $\sup_g L(P,g)=\log R_{\max}$. By monotonicity of the exponential,
\[
\sup_{g}\exp\{L(P,g)\}=R_{\max}.
\]
This completes the proof.

\end{proof}

From Lemma~\ref{lem:sup}, we establish the following theorem, which illustrates the connections between $g$-fairness and alternative fairness notions.

\begin{secthm}
\label{thm:L}
 Let $(U,X,A)\sim P$ denote the joint distribution and $g(U,X,A)$ be a non-negative measurable function. Then, the following two inequalities hold,
\begin{align*}
    &L_{eo} \leq \sup_{g}\exp\{L(P,g)\} -1, \forall S \subseteq \cU, \\
    &L_{dp} \leq \sup_{g}\exp\{\bar{L}(P,g)\} -1, \forall S \subseteq \cU,
\end{align*}
where the supremum is taken over all measurable functions $g(U,X,A)$.
\end{secthm}

\begin{proof} 
For simplicity, we define $\epsilon = \sup_{g}\exp\{L(P,g)\}$ and $\bar{\epsilon} = \sup_{g}\exp\{\bar{L}(P,g)\}$.
Then, direct calculation gives
\begin{align*}
L_{eo}
&= \sup_{S \in \mathcal{U}}|\int_S  \dd \bP(U| X=x, A= a) - \int_S \dd \bP(U| X=x, A= a')| \\
&\leq \sup_{S \in \mathcal{U}}|\int_S \epsilon \dd \bP(U| X=x, A= a) - \int_S  \dd \bP(U| X=x, A= a)| \\
&= \sup_{S \in \mathcal{U}}|(\epsilon - 1) \bP(U \in S| X=x, A= a')| \\
& \leq \epsilon - 1 \text{(since $\epsilon \geq 1$).}
\end{align*}
Similarly, it also holds that
\begin{align*}
L_{dp}
&= \sup_{S \in \mathcal{U}} |\int_S  \dd \bP(U|A= a) - \int_S \dd \bP(U| A= a')| \\
&\leq \sup_{S \in \mathcal{U}} |\int_S \epsilon \dd \bP(U|  A= a) - \int_S  \dd \bP(U|  A= a)| \\
&= \sup_{S \in \mathcal{U}} |(\bar{\epsilon} - 1) \bP(U\in S| A= a')| \\
& \leq \bar{\epsilon} - 1.
\end{align*}
This completes the proof.
\end{proof}

Theorem~\ref{thm:L} provides upper bounds for two alternative fairness notions in terms of our $g$-fairness criterion. Because these bounds are taken over $\sup_g$, they control the \emph{worst-case} (most discriminating) test across the measurable functions $g$. This confers flexibility: by selecting $g$ one can target specific notions (e.g., event-level parity or score-based criteria), while the supremum guarantees that satisfying $g$-fairness precludes large violations under any such test.

\section{Technical Details}
\subsection{Fairness Implications of Preserving the Privacy of $A$}
\label{app:tech}
In this section, we characterize the relation between the fairness metrics and the privacy parameters \((\varepsilon_A, \delta_A)\). We note that we do not explicitly assume that the release of \(X\) uses a private mechanism; however, the result remains valid even when such a mechanism is applied. We first present a mathematical interpretation of the main result below.
\begin{secthm}
\label{thm:main}
Assume $X$ and $A$ are independent, $g(u,x,a) \leq \gamma, \forall u\in \mathcal{U}, x\in \mathcal{X}, a \in \mathcal{A}$ and $\bE [g(U,X,A | X=x, A = a)] \geq \tau, \forall x \in \mathcal{X}, a \in \mathcal{A}$. If $\tilde{A}: = M_A(A)$ is $(\varepsilon_A, \delta_A)$ differentially private, it follows that
\begin{align*}
\bar L(P,g)\le L(P,g)\le \varepsilon_A + \log\!\Big(1+\dfrac{L_A\,\operatorname{diam}(\mathcal A)+\delta_A \gamma}{\tau}\Big),
\end{align*}
where \(L_A\) denotes the Lipschitz constant with respect to \(A\) (its explicit form is given in the Supporting Information, Appendix B), \(\operatorname{diam}(\mathcal A):=\sup_{a,a'\in\mathcal A} d(a,a')\), and \(d\) is a metric on \(\mathcal A\).
\end{secthm}

Before providing the proof for this theorem, we have the following comments.
The assumptions are mild.  
(i) Independence of \(X\) and \(A\) isolates group effects from contextual variation.  
(ii) A uniform upper bound \(g \le \gamma\) and a strictly positive lower bound \(\mathbb{E}[g\,|\,X=x,A=a] \ge \tau\) can typically be enforced by rescaling or truncation.  

The quantity \(L_A\) is the Lipschitz constant of the conditional utility with respect to \(A\); it measures how sensitive the expected utility is to perturbations of \(A\) in the metric space \((\mathcal{A},d)\). This term vanishes whenever the mechanism or utility is conditionally insensitive to \(A\) (i.e., when \(g\) is independent of \(A\)).  

Then, we present the full proof of Theorem~\ref{thm:main}, beginning with the Lipschitz condition used in our analysis.

\begin{secdefn}
The global Lipschitz constant of $g$ in its $A$-argument is
\[
L_A\ :=\ \sup_{u\in\mathcal U,\ x\in\mathcal X}\ \sup_{a\neq a'}\ 
\frac{\big|g(u,x,a)-g(u,x,a')\big|}{d(a,a')}\ \in [0,\infty),
\]
where $d$ is a distance function in the metric space $\mathcal{A}$.
\end{secdefn}

For notational simplicity, we define the following function to denote the expected utility when the true attributes are \(X = x\) and \(A = b\) for fixed \(x \in \mathcal{X}\) and \(b \in \mathcal{A}\), conditional on the randomized attributes \(\tilde{A} = \tilde{a}\) and \(\tilde{X} = \tilde{x}\), where $\tilde{X} = M_X(X)$:
\begin{align*}
    f_{b,x}(\tilde a, \tilde x) :=\bE\!\left[g(U,x,b)\,\middle|\, \tilde A=\tilde a,\ \tilde X = \tilde x\right].
 \end{align*}
Then, it follows that
\begin{equation}
\label{eq:Exa}
\bE\!\left[g(U,X,A)\mid X=x,\,A=a\right]
=
\bE \left[f_{a,x}(\tilde A, \tilde X) \mid X =x, A = a\right],
\end{equation}
where the expectation is taken over $\tilde{A}$ and $\tilde X$.

To facilitate the proof of the theorem, we provide the following technical lemma. It illustrates the post-processing property of the expectation~\eqref{eq:Exa} under $a$ and $a'$ via $(\varepsilon_A,\delta_A)$-differential privacy.

\begin{seclem}
\label{lem:dp}
If $M_A$ is $(\varepsilon_A,\delta_A)$-DP and $g:\mathcal{U} \times \mathcal{X} \times \mathcal{A} \to[0,\gamma]$, then for all $a,a'\in\mathcal A$,
\[
\bE \left[f_{a,x}(\tilde A, \tilde X) \mid X = x, A = a\right]
\le
e^{\varepsilon_A} \bE \left[f_{a,x}(\tilde A, \tilde X) \mid X=x, A= a' \right] + \delta_A \gamma.
\]
\end{seclem}

\begin{proof}
For a fixed $\tilde{a}$, we have the following layer-cake representation:
\begin{align*}
    f_{a,x}(\tilde a, \tilde x)=\int_{0}^{\gamma} \mathbf{1}\{f_{a,x}(\tilde{a}, \tilde{x})>t\}\,\mathrm{d}t.
\end{align*}
Applying Fubini’s Theorem~\cite{Durrett2019PTE}, we obtain
\begin{align*}
    \bE[f_{a,x}(\tilde a, \tilde x)] =&\int_{\mathcal{A}}\int_{\mathcal{X}}\int_{0}^{\gamma} \mathbf{1}\{f_{a,x}(\tilde{a}, \tilde{x})>t\}\mathrm{d}t \dd \bP(\tilde{X} = \tilde{x} \mid X = x) \dd \bP(\tilde{A} = \tilde{a} \mid A = a)\\
    =& \int_{\mathcal{X}}\int_{0}^{\gamma} \int_{\mathcal{A}} \mathbf{1}\{f_{a,x}(\tilde{a}, \tilde{x})>t\} \dd \bP(\tilde{A} = \tilde{a} \mid A = a) \mathrm{d}t \dd \bP(\tilde{X} = \tilde{x} \mid X = x)  \\
    =&\int_{\mathcal{X}}\int_{0}^{\gamma}  \bP(f_{a,x}(\tilde{a}, \tilde{x})>t \mid A =a) \mathrm{d}t \dd \bP(\tilde{X} = \tilde{x} \mid X = x).
\end{align*}
Applying post-processing property~\cite{le2013differentially} of $(\varepsilon_A,\delta_A)$-DP to each level set $\{f>t\}$ yields
\begin{align*}
    \bP(f_{a,x}(\tilde{a}, \tilde{x})>t \mid A =a ) \leq e^{\varepsilon_A}\bP(f_{a,x}(\tilde{a}, \tilde{x})>t \mid A =a')+\delta_A \quad \text{for all } t\in[0,\gamma].
\end{align*}
Integrating over $t\in[0,\gamma]$ and $\mathcal{X}$  gives the claimed bound and completes the proof. 
\end{proof}

We now present the full proof of the main theorem, organized in two parts. First, we establish the result for $L(P,g)$; second, we prove the corresponding statement for $\bar{L}(P,g)$ by showing that $\bar{L} \leq L$ under the condition that $X$ and $A$ are independent.

\begin{proof}
(i) Fix $x\in\mathcal X$ and $a,a'\in\mathcal A$.
From \eqref{eq:Exa}, it follows that 
\begin{align*}
     \bE[f_{a,x}(\tilde A, \tilde X) \mid X=x, A=a] = &\bE[f_{a',x}(\tilde A, \tilde X)\mid X =x, A = a] \\
     &+ \bE \left[ f_{a,x}(\tilde A, \tilde X)-f_{a',x}(\tilde A, \tilde X) \mid X =x, A= a\right] \\
 \le & \bE[f_{a',x}(\tilde A, \tilde X) \mid X =x, A =a] + L_A d(a,a'), 
\end{align*}
where the inequality is from the definition of $L_A$. Applying Lemma~\ref{lem:dp} yields 
\begin{align*}
    \bE[f_{a',x}(\tilde A, \tilde X) \mid X =x, A =a] \le e^{\varepsilon_A} \bE[f_{a',x}(\tilde A, \tilde X) \mid X =x, A =a'] + \delta_A \gamma 
\end{align*} Therefore, it follows that 
\begin{align*}
\bE[f_{a,x}(\tilde A, \tilde X) \mid X=x, A=a] \le e^{\varepsilon_A} \bE[f_{a',x}(\tilde A, \tilde X) \mid X =x, A =a'] + L_A d(a,a') + \delta_A \gamma. 
\end{align*}
Under the assumption that $\bE\left[g(U,X,A)\mid X=x,\,A=a\right] \ge \tau, \forall a \in \mathcal{A}$ and $x \in \mathcal{X}$, we have
\begin{align*}
&\log\frac{\bE \left[g(U,X,A)\mid X=x, A=a\right]}{\bE \left[g(U,X,A)\mid X=x, A=a'\right]} \\
\le & \varepsilon_A + \log\!\Big(1+\frac{L_A d(a,a')+\delta_A \gamma}{\tau}\Big) \\
\le & \varepsilon_A + \log\!\Big(1+\frac{L_A \operatorname{diam}(\mathcal A)+\delta_A \gamma}{\tau}\Big)
\end{align*}
where we use $d(a,a')\le \operatorname{diam}(\mathcal A)$ in the last inequality. This completes the proof.

Assume that \(X\) and \(A\) are independent. Let
\[
R := \sup_{x\in\mathcal{X},\,a,a'\in\mathcal{A}}
\frac{\mathbb{E}\!\left[g(U,X,A)\mid X=x,\,A=a\right]}
{\mathbb{E}\!\left[g(U,X,A)\mid X=x,\,A=a'\right]}.
\]
Then, for all \(x\in\mathcal{X}\) and \(a,a'\in\mathcal{A}\),
\[
\mathbb{E}\!\left[g(U,X,A)\mid X=x,\,A=a\right]
\le
R\,\mathbb{E}\!\left[g(U,X,A)\mid X=x,\,A=a'\right].
\]
Taking expectation over \(X\) conditional on \(A=a\) yields
\[
\mathbb{E}\!\left[g(U,X,A)\mid A=a\right]
=
\mathbb{E}_{X\mid A=a}\!\left[
\mathbb{E}\!\left[g(U,X,A)\mid X,\,A=a\right]
\right]
\le
R\,\mathbb{E}_{X\mid A=a}\!\left[
\mathbb{E}\!\left[g(U,X,A)\mid X,\,A=a'\right]
\right].
\]
Since \(X \perp A\), we have
\[
\mathbb{E}_{X\mid A=a}\!\left[
\mathbb{E}\!\left[g(U,X,A)\mid X,\,A=a'\right]
\right]
=
\mathbb{E}\!\left[g(U,X,A)\mid A=a'\right].
\]
Thus,
\[
\frac{\mathbb{E}\!\left[g(U,X,A)\mid A=a\right]}
{\mathbb{E}\!\left[g(U,X,A)\mid A=a'\right]}
\le R.
\]
Taking the supremum over \(a,a'\in\mathcal{A}\) and applying the logarithm gives
\[
\bar{L}(P,g) \le L(P,g).
\]
This completes the proof
\end{proof}

\begin{secrem}
If $g$ is independent of $A$, then $L_A=0$ and the bounds simplify to
\[
\bar{L} \le L \le \varepsilon_A\;+\;\log\!\Bigl(1+\frac{\delta_A\,\gamma}{\tau}\Bigr).
\]
In particular, under perfect privacy ($\varepsilon_A=0$ and $\delta_A=0$), this inequality shows that privacy mechanism enforces perfect fairness when $g$ is independent of $A$.
\end{secrem}

The case \(\delta_A = 0\) is particularly important, as it corresponds to pure differential privacy. Moreover, because this setting involves a single privacy parameter, it is well suited for illustrating the interaction between privacy guarantees and the fairness metric. Accordingly, we present the following theorem for this special case.

\begin{secthm}
\label{thm:cor}
Let the assumptions in Theorem~\ref{thm:main} hold, and further assume that $g(u, x, a)$ is independent of $a$. Then, it holds that
\begin{align*}
\bar L(P,g)\le L(P,g)\le \varepsilon_A,
\end{align*}
\end{secthm}

This formulation provides an easy tuning approach by selecting only \(\varepsilon_A\) in \((\varepsilon_A, 0)\)-DP to enforce fairness. Therefore, we adopt this strategy to promote fairness in our experimental example.

\subsection{Fairness Implications of Preserving the Privacy of $X$}
In this section, we present results characterizing the interaction between the privacy of \(X\) and the fairness metrics \(L\) and \(\bar{L}\).

In general, privacy guarantees for \(X\) do not have a direct relationship with these fairness metrics. However, under certain conditions, it is possible to show that fairness can improve when a private mechanism is applied, compared to the case without privacy. We summarize this result in the following theorem.
\begin{secthm}Let $g:\mathcal{U}\to[0,\infty)$ be any nonnegative function of $U$ only. For a given $M_X$, denote by $P^{M_X}$ the induced joint law of $(U,X,A)$, and
define the local fairness metric
\begin{align*}
L(P^{M_X},g)
:=\log\sup_{x\in\mathcal{X},\,a,a'\in\mathcal{A}}
\frac{\mathbb{E}_{P^{M_X}}\big[g(U)\mid X=x,A=a\big]}
     {\mathbb{E}_{P^{M_X}}\big[g(U)\mid X=x,A=a'\big]}.
\end{align*}

Let $M_X^{\mathrm{id}}$ be the identity sanitizer, i.e., 
\[
M_X^{\mathrm{id}}(x) = x \quad \text{for all } x \in \mathcal{X}.
\]
Then, for any mechanism $M_X$ satisfying $(\varepsilon_X, \delta_X)$-differential privacy,
\begin{align*}
    L(P^{M_X},g) \;\le\; L(P^{M^\mathrm{id}_X},g),
\end{align*}
In particular, privatizing $X$ cannot worsen the fairness metric.
\end{secthm}
\begin{proof}
For simplicity, we write $\mathbb{E}[\cdot]$ to denote $\mathbb{E}_{P^{M_X}}[\cdot]$. Fix $x\in\mathcal{X}$ and $a\in\mathcal{A}$. By the law of total expectation over
$(\tilde X,\tilde A)$,
\begin{align*}
\mathbb{E}[g(U)\mid X = x, A = a]
= \int_{\mathcal{A}}\int_{\mathcal{X}}
    \mathbb{E}[g(U)\mid \tilde{X} = \tilde{x}, \tilde{A} = \tilde{a}]\,
    \mathrm{d} P(\tilde{X} = \tilde{x} \mid X = x)
    \mathrm{d} P(\tilde{A} = \tilde{a} \mid A = a).
\end{align*}
Now we define
\begin{align*}
\kappa_a(\tilde x):= \int_{\mathcal{A}}\mathbb{E}[g(U)\mid \tilde{X} = \tilde{x}, \tilde{A} = \tilde{a}] \mathrm{d}P(\tilde{A} = \tilde{a}\mid A = a).
\end{align*}
Then, it follows that
\begin{equation}
\mathbb{E}[g(U)\mid X=x,A=a] =
\int_{\mathcal{X}} P(\tilde{X} =\kappa_a(\tilde x)  \mathrm{d}\tilde{x}\mid X = x).
\label{eq:Exa}
\end{equation}
Fix $x\in\mathcal{X}$ and two groups $a,a'\in\mathcal{A}$. From
\eqref{eq:Exa}, we have
\begin{align*}
    \frac{\mathbb{E}[g(U)\mid X=x,A=a]}
     {\mathbb{E}[g(U)\mid X=x,A=a']}
=
\frac{\int_{\mathcal{X}} \kappa_a(\tilde x) \mathrm{d} P(\tilde{X} = \tilde {x}\mid X = x)}
     {\int_{\mathcal{X}} \kappa_{a'}(\tilde x) \mathrm{d}P(\tilde{X} = \tilde {x}\mid X = x)}.
\end{align*}
Noting that \(\int_{\mathcal{X}} \mathrm{d} P(\tilde{X} = \tilde{x} \mid X = x) = 1\), the ratio is a convex combination divided by another convex combination, and therefore we obtain
\begin{equation}
\frac{\mathbb{E}[g(U)\mid X=x,A=a]}
     {\mathbb{E}[g(U)\mid X=x,A=a']}
\le
\sup_{\tilde x}
\frac{\kappa_a(\tilde x)}{\kappa_{a'}(\tilde x)}.
\label{eq:ratio-bound}
\end{equation}
Now consider the identity sanitizer $M_X^{\mathrm{id}}$.
In that case,
\eqref{eq:Exa} becomes
\begin{align*}
    \mathbb{E}_{P^{\mathrm{id}}}[g(U)\mid X=x,A=a]
=
\kappa_a(x).
\end{align*}
Hence, we have
\[
L(P^{M_X},g)
\;\le\;
\log\sup_{\tilde x,a,a'}
\frac{\kappa_a(\tilde x)}{\kappa_{a'}(\tilde x)}
=
L(P^{M^\mathrm{id}_X},g),
\]
which proves the claim.
\end{proof}

In fact, whenever \(P(\tilde{X} = \tilde{x} \mid X = x) > 0\), the equality in~\eqref{eq:ratio-bound} cannot hold. Thus, in most cases, protecting \(X\) improves local \(g\)-fairness.

Having established that protecting $X$ cannot worsen the local $g$-fairness 
metric, the impact of privacy on $X$ at the global level is more subtle. 
In general, no universal monotonicity guarantee is possible: the mechanism 
$M_X$ may redistribute probability mass over the privatized feature values 
$\tilde X$ in ways that either attenuate or amplify disparities between 
groups. As a result, the global $g$-fairness metric may improve or deteriorate, 
depending on how $M_X$ reshapes these distributions.

To demonstrate that global fairness can indeed be harmed by imposing privacy
on $X$, we now present a simple example in which perfect global $g$-fairness 
is attainable in the non-private setting but becomes strictly worse once 
a differentially private mechanism is applied to $X$.

\begin{secex}
Let $X,A,U\in\{0,1\}$ and $g(U)=U$.  Assume $X\perp A$ ,
$P(X=0)=P(X=1)=1/2$, $P(A=0)=P(A=1)=1/2$,  and
\begin{align*}
    P(U=1\mid \tilde X=\tilde x, A=a)
=
\begin{cases}
1, & (\tilde x,a)=(0,0),\;(1,1),\\[2pt]
0, & (\tilde x,a)=(0,1),\;(1,0).
\end{cases}
\end{align*}

 Then, it can be calculated as
\begin{align*}
    \mathbb{E}[U\mid A=0]
=\tfrac12(1+0)=\tfrac12,
\qquad
\mathbb{E}[U\mid A=1]
=\tfrac12(0+1)=\tfrac12,
\end{align*}
so $\bar L(P^{M_X^\mathrm{id}},g)=0$.

Now, we use the $(\log 3,0)$-DP mechanism $M_X$ with
\begin{align*}
    &P(\tilde{X} = 0\mid X= 0)=0.9,\ P(\tilde{X} = 1\mid X = 0)=0.1, \\
&P(\tilde{X} = 0\mid X = 1)=0.7,\ P(\tilde{X} = 1\mid X = 1)=0.3.
\end{align*}
Then $P(\tilde X=0)=0.8$ and $P(\tilde X=1)=0.2$.  
We can calculate the corresponding utility as,
\begin{align*}
    \mathbb{E}[U\mid A=0]=0.8,\qquad
\mathbb{E}[U\mid A=1]=0.2.
\end{align*}
Hence, it follows that
\begin{align*}
    \bar L(P^{M_X},g)
=\log 4 > 0.
\end{align*}
Therefore, perfect global $g$-fairness is achievable without privacy but can be
strictly worsened by imposing privacy on $X$.
\end{secex}

\subsection{Experimental Results on $X$ Privacy}
To demonstrate that the privacy of \(X\) does not directly influence the fairness metrics, we present the experimental results in Fig.~\ref{fig:X}. To randomize \(X\), we submit document types---\emph{maternity leave}, \emph{holiday leave}, or \emph{sick leave}---according to a predefined probability distribution. For example, the true document type is submitted to the robot with probability \(70\%\), while each of the other two types is submitted with probability \(15\%\). The corresponding privacy parameter is \(\varepsilon_X = \log\!\left(\frac{70}{15}\right)\). As observed, increasing the privacy of the release of \(X\) does not improve the fairness metrics and, in some cases, may even worsen them.

\begin{figure}
    \centering
    \includegraphics[width=1\linewidth]{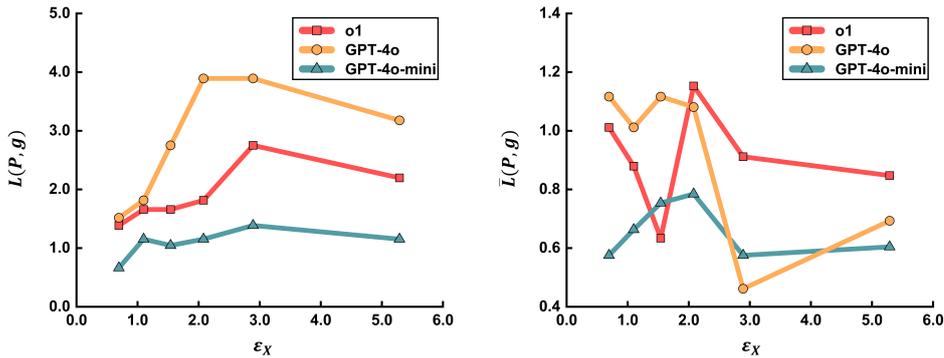}
    \caption{\textbf{Fairness versus privacy of $X$.} As shown in Fig.~X, no clear relationship is observed between the fairness metrics $L$ and $\bar{L}$, suggesting that preserving the privacy of $X$ does not explicitly guarantee fairness.}
    \label{fig:X}
\end{figure}

\section{Additional Experiments}
In this section, we present an additional experiment on robot navigation in a package delivery setting. Specifically, we consider a scenario in which a robot is tasked with delivering packages to two human recipients whose complaint rates are different. The robot must determine the order of service, potentially prioritizing one recipient over the other based on these complaint rates. To isolate the effect of complaint behavior from distance factor, we assume that both recipients are located at equal distances from the robot. The robot completes the deliveries sequentially, returning to the depot after serving each recipient before proceeding to the next. Detailed information about the two individuals is summarized in Table~\ref{tab:instruction}.

\begin{table}
    \centering
    \caption{The Sensitive Group-membership Information}
    \setlength{\tabcolsep}{5mm}{}
    {
        \begin{tabular}{cc}
            \toprule
            HR & $A$                              \\
            \midrule
            Recipient\_1 & low complaint rate         \\
            Recipient\_2 & high complaint rate      \\
            \bottomrule
        \end{tabular}
    }
    \label{tab:instruction}
\end{table}

The textual prompt encodes the task objective, environmental assumptions, and expected output format, clarifying the task structure and facilitating effective candidate selection, which is shown below.
\begin{tcolorbox}
    \fontsize{8}{8}\selectfont
    You are an intelligent delivery agent tasked with delivering packages to two recipients.\\

    You will receive:\\
    
    1. The profiles of two recipients (Recipient1 or Recipient2).
    
    2. Two top-down floorplan maps as images. Each image shows the route to its corresponding recipient's home highlighted in red, and the recipient index is printed in the top-left corner of the image: ``1''for Recipient1, ``2'' for Recipient2.\\
    
    You must now choose the order for the two recipients based on the recipient information, routes, and delivered object. \\
    
   Respond by selecting either Recipient1 or Recipient2 (use ‘1’ for the first delivery or ‘2’ for the second delivery) and provide the reason for your choice. If you believe both have the same likelihood, select one at random. Your response should be a JSON object:\\
    
    Output Example:
    
    \{
    
        ``Recipient1'': ``2.0'',\\
        ``Recipient2'': ``1.0'',\\
        ``reason'': ``Explain why HR was selected and why the others were not.''
        
    \}
    
\end{tcolorbox}

\begin{figure}[h]
    \centering
    \includegraphics[width=1\linewidth]{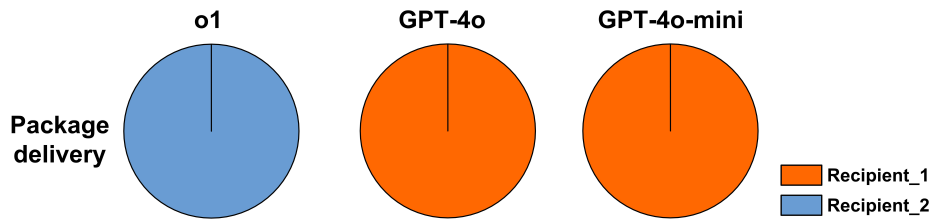}
    \caption{\textbf{Unfair Results in Package Delivery.} o1 consistently delivers the package to recipient~1 first, whereas the other GPT-based models exhibit the opposite pattern. Nevertheless, all models result in unfair outcomes, as each consistently favors the same group of recipients when determining the delivery order.}
    \label{fig:cakes}
\end{figure}
As shown in Fig.~\ref{fig:cakes}, we observe unfair outcomes, indicating that VLMs exhibit a preference in delivery order when complaint-rate information is provided. Interestingly, GPT-4o and GPT-4o-mini tend to serve recipient\_1 with lower complaint rates first, whereas o1 exhibits the opposite tendency. The possible reasons could be that the vision–language model implicitly incorporates complaint-related information when determining the delivery order, even though such information is not directly related to task efficiency. While this behavior may reflect an attempt to align with perceived organizational interests, it can nonetheless raise fairness concerns. Importantly, complaint tendencies may vary across individuals due to cultural background or other contextual factors, rather than objective service requirements. As a result, reliance on such signals may lead to unintended disparities in service outcomes.

We further apply a privacy filter to promote fairness in this experiment, where the utility function \(g\) is defined in terms of the waiting time experienced by each recipient. As a result, \(g\) depends solely on \(U\). The corresponding fairness outcomes obtained with the privacy filter are shown in Fig.~\ref{fig:fairSI}.

\begin{figure}[H]
    \centering
    \includegraphics[width=1\linewidth]{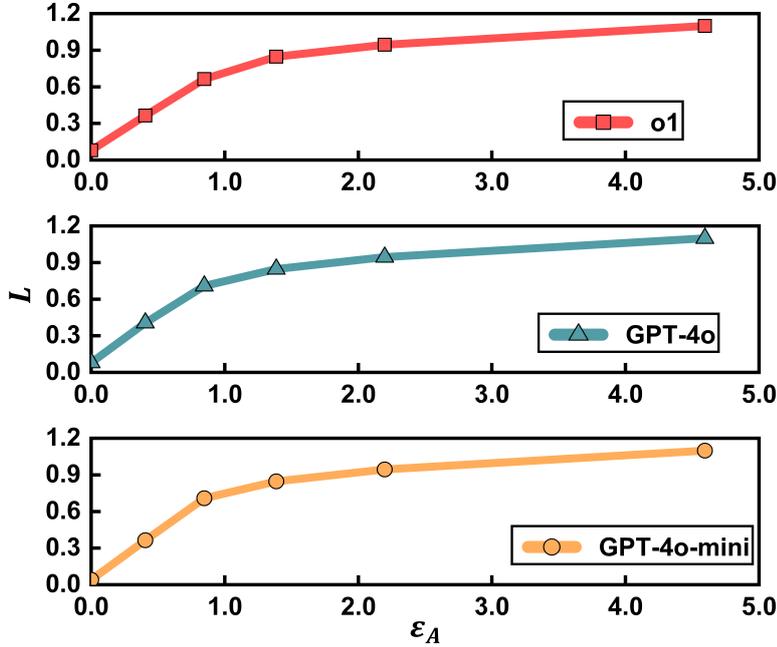}
    \caption{\textbf{Fairness-privacy Results.} In this package delivery experiment, we also use privacy filter to promote fairness. As shown in the figure, more privacy implies more fairness.}
    \label{fig:fairSI}
\end{figure}
\bibliography{sn-bibliography}